\documentclass[letterpaper, 10 pt, conference]{ieeeconf}  

\IEEEoverridecommandlockouts                              
\usepackage{ifthen}
\newboolean{showcomments}
\setboolean{showcomments}{true}
\usepackage[usenames,dvipsnames]{color}
\usepackage{todonotes}

\definecolor{bleudefrance}{rgb}{0.19, 0.55, 0.91}
\definecolor{ao(english)}{rgb}{0.0, 0.5, 0.0}

\newcommand{\addcite}[0]{\ifthenelse{\boolean{showcomments}}
{\textcolor{purple}{(add cite(s)) }}{}}%

\newcommand{\enrique}[1]{  \ifthenelse{\boolean{showcomments}}
{\todo[inline,color=bleudefrance]{Enrique: #1}}{}}
\newcommand{\emmargin}[1]{\ifthenelse{\boolean{showcomments}}{\marginpar{\color{bleudefrance}\tiny EM: #1}}{}}

\newboolean{showedits}
\setboolean{showedits}{false}
\usepackage[markup=underlined]{changes}
\definechangesauthor[color=bleudefrance]{EM}
\newcommand{\aem}[1]{
\ifthenelse{\boolean{showedits}}
{\added[id=EM]{#1}}
{\!#1\hspace{-4.75pt}}
}
\newcommand{\repem}[2]{
\ifthenelse{\boolean{showedits}}
{\replaced[id=EM]{#1}{#2}}
{\!#1\hspace{-4.75pt}}
}
\newcommand{\dem}[1]{
\ifthenelse{\boolean{showedits}}
{\deleted[id=EM]{#1}}
{}
}

\usepackage{amsmath} 
\allowdisplaybreaks
\usepackage{amssymb}  

\usepackage{amsthm}
\usepackage{optidef}
\usepackage[ruled]{algorithm2e}
\newtheorem{lemma}{Lemma}
\newtheorem{theorem}{Theorem}
\usepackage{graphicx}
\usepackage{dsfont}
\usepackage{commath}
\usepackage[noadjust]{cite}

\usepackage[american]{circuitikz}
\usepackage{tikz}
\usepackage[caption=false, font=footnotesize]{subfig}

\theoremstyle{definition}

\newtheorem*{remark}{Remark}

\DeclareMathOperator{\handicap}{Handicap}
\usepackage{setspace}
\title{\LARGE \bf
Learning to be safe, in finite time
}

\author{Agustin Castellano,$^1$ Juan Bazerque,$^1$ and Enrique Mallada$^2$
\thanks{$^{1}$Agustin Castellano and Juan Bazerque are with the Universidad de la Republica, Montevideo, Uruguay.
        \texttt{\{acastellano, jbazerque\}@fing.edu.uy}.
$^{2}$Enrique Mallada is with the Johns Hopkins University, Baltimore, Maryland, USA.     \texttt{mallada@jhu.edu}
The work was supported by NSF through grants CNS 1544771, EPCN 1711188, AMPS 1736448 and CAREER 1752362, and by ANII-Uruguay through grant FSE-1-2019-1-159457.}
}

\begin{document}

\maketitle
\thispagestyle{empty}
\pagestyle{empty}

\begin{abstract}
This paper aims to put forward the concept that learning to take safe actions in unknown environments, even with probability one guarantees, can be achieved without the need for an unbounded number of exploratory trials, provided that one is willing to relax its optimality requirements mildly. We focus on the canonical multi-armed bandit problem and seek to study the exploration-preservation trade-off intrinsic within safe learning. More precisely, by defining a 
handicap metric that counts the number of unsafe actions, we provide an algorithm for discarding unsafe machines (or actions), with probability one, that achieves constant handicap.
Our algorithm is rooted in the classical sequential probability ratio test, redefined here for continuing tasks.  Under standard assumptions on sufficient exploration, our rule provably detects all unsafe machines in an (expected) finite number of rounds. The analysis also unveils a trade-off between the number of rounds needed to secure the environment and the probability of discarding safe machines. Our decision rule can wrap around any other algorithm to optimize a specific auxiliary goal since it provides a safe environment to search for (approximately) optimal policies. Simulations corroborate our theoretical findings and further illustrate the aforementioned trade-offs.

\end{abstract}

\section{Introduction}

Learning to take \emph{safe actions} in unknown environments is a general goal that spans across multiple disciplines. Within control theory safety is intrinsic to robust analysis and design~\cite{zhou1996robust}, where controllers, with stability and performance guarantees,  are designed for uncertain systems. It is also at the core of statistical decision theory~\cite{berger2013statistical}, where the inference and the decision-making processes are intertwined towards the common goal of making accurate decisions based on limited information. Though, historically, these two disciplines have been deemed as seemingly disconnected, such separation is rapidly vanishing.


Motivated by the success of machine learning in achieving super human performance, e.g., in vision~\cite{krizhevsky2012imagenet,rawat2017deep}, speech~\cite{hinton2012deep,graves2013speech}, and video games~\cite{silver2016mastering,vinyals2017starcraft}, there has been recent interest in developing learning-enabled technology that can implement highly complex actions for safety-critical autonomous systems, such as self-driving cars, robots, etc. However, without proper safety guarantees such systems will rarely be deployed. There is therefore the need to develop analysis tools and algorithms that can provide such guarantees during, and after, training.

A common approach to solve this problem is to search for actions, policies, or controllers that optimize a cost or reward subject to safety requirements imposed as constraints. Examples include, adding safe constraints for reinforcement learning algorithms~\cite{2006.Geibel,2019.Zanon,2019.Paternainwed,2019.Chengfp9,2017.Achiam}, (robust) stability constraints to  learning algorithms~\cite{2020.Dean,Dean,2019.Fazlyab9wc}, and solving constrained multi-armed bandits~\cite{2020.Moradipari,2019.Amani}. When such constraints are being included during training, algorithms that converge to optimal policies,  guarantee safety asymptotically. However, such an approach  fails to provide guarantees while learning.

In this paper, we suggest an alternative approach. Instead of focusing on finding optimal actions subject to, a priori unknown, safety constrains, we argue that one should tackle the problem of learning safe actions separately and more efficiently. To illustrate this point, we study the problem of finding safe actions within the canonical setting of the multi-armed bandit (MAB) problem. This setting lacks state transitions and is thus an abridged version of learning in dynamical environments. In the MAB setting, one  is given a set of $N$ machines/actions with expected reward $\mu_n\in[0,1]$ for $n\in\{1,\dots,N\}$. 
By letting safe actions to be those who choose machines with reward larger than some nominal $\mu$, we define a \emph{handicap} metric ---akin to regret--- that counts the number of times unsafe actions are chosen. 
Leveraging classical results on sequential hypothesis tests \cite{wald1945}, we provide an algorithm for detecting unsafe machines. Unlike the regret minimization counterpart of this problem, which requires an unbounded (with logarithmic growth) number of trials  of sub-optimal actions to identify the machine with highest reward~\cite{1985.Lai}, our algorithm  discards all unsafe machines with only {a finite number of trials of unsafe actions. By characterizing such number, we  guarantee that the  total handicap remains bounded by a constant.} 

More precisely, we use a modified version of Wald's sequential probability ratio test (SPRT)~\cite{wald1945} for each machine, with the null ($H_0$) and alternative ($H_1$) hypothesis  being that the machine is safe and unsafe, respectively. {Intentionally different than the SPRT, which  aims to decide  either $H_0$ or $H_1$  and then stop,  our goal is to discard  unsafe machines only by deciding on $H_1$.  }
This allows our algorithm to \emph{identify all unsafe machines with probability one} within a finite expected number of trials.  Our analysis also unveils an \emph{exploration-preservation trade-off} between the false-positive ratio (safe machines discarded) and the total handicap experienced (number of trials  on unsafe machines).
Notably, our decision rule can further wrap around any other algorithm to optimize a specific auxiliary goal since it provides a safe environment to search for (approximately) optimal policies.


The rest of the paper is organized as follows. We introduce our problem setup in Section \ref{sec:problem_Statement}. For didactic purposes, we first look at the case where we aim to find machines with $\mu=1$ in  Section \ref{sec:flawless}, construct a sequential test that extends this case for one machine in Section \ref{sec:sprt}, and generalize the solution in Section \ref{sec:relaxed}. Numerical illustrations are provided in Section \ref{sec:numerical} and we conclude in Section \ref{sec:conclusions}.

\section{Problem statement}\label{sec:problem_Statement}
Consider the  setup  of a multi-armed bandit problem, in which at each time instant $t=1,2,\ldots$  we have the choice to operate one out of  $N$ machines. If machine $n\in \{1,2,\ldots,N\}$ is operated at time $t$, it returns a binary value $X_{n,t}$ which is modelled as a   Bernoulli random variable with parameter $\mu_n$. This return reveals whether the  action  led to a safe result in which case  $X_{n,t}=1$, or an unsafe one if $X_{n,t}=0$. A machine is said to be safe if its operation leads to a safe result. In this sense we consider two cases, one in which we only accept flawless machines, i.e., those with $\mu_n=1$, and a relaxed condition in which a machine is defined to be safe if  $\mu_n\geq \mu$, with $\mu\in(0,1)$ being a  prescribed safety requirement (possibly $\mu\simeq 1$). 

Let $I_t\in \{1,2,\ldots,N\}$ denote the index of the machine selected at time $t$, and $X_t\doteq X_{I_t,t}$ the corresponding return. Our goal is to design a selection policy and a decision  rule that uses data $X_t,\ t=1,2,\ldots$ to remove all unsafe machines, while guaranteeing that a prescribed proportion of the safe machines are kept. 
If only flawless machines  are accepted then the solution is straightforward: the algorithm should  remove  machines as soon as they return their first $X_{t}=0$. We will analyse this case first in Section \ref{sec:flawless}.

For the relaxed condition, we will develop a one-sided Sequential Probability Ratio Test (SPRT) that  removes unsafe machines with $\mu_n\leq \mu$  almost surely.  In order to guarantee  that unsafe machines are removed in finite time, and provide an explicit bound on the expected number of trials needed, we need to sacrifice a proportion of the safe machines. For this purpose we prescribe a slack parameter $\epsilon$ and a probability $\alpha$, and show that a proportion $1-\alpha$ of those safe machines with $\mu_n\geq \mu+\epsilon$ are kept as $t\to\infty$. 
We will develop this modified version of the SPRT in Section \ref{sec:sprt} for the case of $N=1$, extending it in Section \ref{sec:relaxed} to the multi-armed bandit setup.

Along the way, we will introduce three figures of merits that are instrumental to goal of learning to be safe. One is the \emph{handicap}, that complements the idea of regret for operating unsafe machines, and counts the number of unsafe actions chosen so far.  {Closely related to the notion of handicap is the testing time, that counts the number of times a machine is tried for safety, and  is related to the detection time of unsafe machines.} The third one is the safety ratio, which counts the proportion of safe machines that are kept at time $t$. 

\section{Safe learning with flawless machines}\label{sec:flawless}
Consider the multi armed bandit setup described above with $N$ machines, $M$ of them unsafe or \emph{malfunctioning}. In order to simplify notation and without loss of generality, we assume that the first $M$ machines are unsafe so that $\mu_n<1$  for $n=1,\ldots,M$, and $\mu_n=1$ for $n=M+1,\ldots,N$.

We are assured that $X_{n,t}=1~\forall t$ if the machine is safe, thus we can discard those machines that return $X_{n,t}=0$. This is the strategy in Algorithm \ref{alg:assured_explorer}, which selects actions at random over the set $\mathcal S_t$ of  machines that remain at time $t$. 
\begin{algorithm}
Initialize $\mathcal{S}_1=\left\{1,\ldots,N\right\}$\\
\For{$t=1,2,\ldots,$}{
Pick an arm $I_t \sim$ Unif($\mathcal{S}_t$)\\
Observe return $X_t=X_{I_t,t}$\\
\If{$X_t=0$}{$\mathcal{S}_t \leftarrow \mathcal{S}_{t-1}\setminus \{I_t\}$
}
\Else{$\mathcal{S}_t\leftarrow\mathcal{S}_{t-1}\;$}}\caption{Safety Inspector}
\label{alg:assured_explorer}
\end{algorithm}

The following definitions are introduced for the purpose of analysing  the Safety Inspector Algorithm \ref{alg:assured_explorer} and its relaxed version in  Section \ref{sec:relaxed}. 
First, even if we recognize that detecting  unsafe machines requires unsafe actions to be taken,  we want to measure if our algorithms pick those unsafe machines efficiently. For this purpose we present the notion of \emph{handicap}, defined as the number of times an unsafe machine is selected, i.e.,
\begin{equation}
    \handicap_t=t-\sum_{\tau=1}^t \mathds{1}\{\mu_{I_\tau}\geq\mu\}\label{eq:handicap}
\end{equation}
where $\mu=1$ in this section, and $\mathds 1(\cdot)$ represents the indicator function which returns one or zero when its argument is true or false, respectively.

{
\begin{remark}
We use the word \emph{handicap} in the sense of a measure of \emph{``a disadvantage that makes achievement unusually difficult,''} as its definition suggests~\cite{mw:handicap}.
An algorithm with unbounded $\handicap_t$ takes unsafe actions infinitely often and is  prone to malfunctioning.
This marks a stark contrast with the notion of \emph{regret}, typically studied in Bandit settings~\cite{lattimore2020bandit}, where unbounded regret is unavoidable~\cite{1985.Lai}.
\end{remark}
}

Notice that if machine  $n$ is selected at time $\tau$, then  $\mu_{I_\tau}=\mu_n$. Thus, the indicator function will return $1$ at time $\tau$ only when a flawless machine is selected. Furthermore,  even if $\mu_n$ is deterministic, $\handicap_t$ is still a random variable, with  randomness coming from the selection $I_t$.
Even if $I_t$ is selected in  round-robin instead of uniformly as in Algorithm \ref{alg:assured_explorer}, the set $\mathcal{S}_t$ is conditioned on previous instances of $X_t$, which are uncertain. In light of this, it is noticeable that an algorithm with low handicap is one which selects unsafe machines infrequently. 
As a second figure of merit we define the safety ratio $\rho_t$ as the proportion of safe machines that are kept after $t$ time slots, i.e.,
\begin{equation}
\rho_t=\frac{\sum_{n\in\mathcal S_t } \mathds 1(\mu_n\geq \mu+\epsilon)}{N-M}\label{eq:safety_ratio}
\end{equation}
with $\epsilon=0$ for Algorithm \ref{alg:assured_explorer}.
Together with the notions of handicap and safety ratio, we are interested in analysing the time that elapses until an unsafe machine is removed.  For this purpose it is instrumental to define the number of times that machine $n$ has been tested for safety after $t$ iterations of Algorithm $\ref{alg:assured_explorer}$, that is 
\begin{equation}T_n(t)=\sum_{\tau=1}^t \mathds{1}(I_\tau=n).
\label{eq:tnt}\end{equation} 

Next, we present a lemma that links the definitions of  $\handicap_t$ with $T_n(t)$, and will be useful to bound the expected handicap for Algorithm \ref{alg:assured_explorer} and that in   Section \ref{sec:relaxed}.

\begin{lemma}
 $\mathbb{E}[{\emph{Handicap}_t}]=\sum_{n=1}^M\mathbb{E}[T_n(t)]$\label{lemma:handicap_tnt}
\end{lemma}
\begin{proof}
\begin{align*}
    &\mathbb{E}[\handicap_t] =  \sum_{\tau=1}^t \mathbb{E}\left[\mathds{1}\{\mu_{I_\tau}<\mu\}\right]\\
    &= \sum_{\tau=1}^t\sum_{n=1}^N P(I_\tau=n) \mathds{1}\{\mu_{I_\tau}<\mu\}
    = \sum_{\tau=1}^t\sum_{n=1}^M P(I_\tau=n)\\
    &= \sum_{n=1}^M\sum_{\tau=1}^t\mathbb{E}\left[\mathds{1}\{I_\tau=n\}\right]
    =\sum_{n=1}^M\mathbb{E}[T_n(t)]
\end{align*}
\end{proof}

Using the result in Lemma \ref{lemma:handicap_tnt} we can bound the expected handicap of Algorithm \ref{alg:assured_explorer} by bounding  $\mathbb{E}[T_n(t)]$. This is the result of the next Theorem.

\begin{theorem}
 The handicap and safety ratio of Algorithm \ref{alg:assured_explorer} satisfy
\begin{align}
\mathbb{E}[\emph{Handicap}_t]&\leq \sum_{n=1}^M \frac{1}{(1-\mu_n)}\\
    \mathbb{E}[\rho_t]&=1,
\end{align}
and the testing time of unsafe machines is bounded by
\begin{align}
\mathbb{E}[T_n(t)]&\leq\frac{1}{(1-\mu_n)}\label{eq:testing_time}
\end{align}
\label{prop:assuredRL}
\end{theorem}
\begin{proof}
The result for the safety ratio is straightforward since the probability of removing a machine with $\mu_n=1$ is zero.
The  bound for the handicap follows from Lemma \ref{lemma:handicap_tnt}, together with the result for the  testing time, which is proved next 
\begin{align}
\mathbb{E}[T_n(t)]&=\sum_{\tau=1}^t\text{P}(T_n(t)=\tau)\tau=\sum_{\tau=1}^t\mu_n^{\tau-1}(1-\mu_n)\;\tau\label{eq:testing time}\\
&\leq\sum_{\tau=1}^\infty \mu_n^{\tau-1}(1-\mu_n)\;\tau = \frac{1}{1-\mu_n}
\label{eq:ExpectedT}
\end{align}
\end{proof}

\begin{remark}
The right hand side of \eqref{eq:testing_time} in Theorem \ref{prop:assuredRL} bounds the expected number of  times that an unsafe machine is tested before removing it. This  highlights one of the main ideas introduced in this paper: if we only  want to detect unsafe machines instead of estimating the exact value of $\mu_n$, then we can do it in finite time. As a consequence, the measure of handicap defined in \eqref{eq:handicap} remains bounded by a constant. Even if we deem this result as conceptually relevant, it presents the drawback that the bounds for the expected handicap and testing times are given in terms of $\mu_n$ which are unknown. In order to provide an explicit bound in terms of the design parameters of the algorithm it is convenient to relax the condition that defines a safe machine, allowing for machines with $\mu_n\geq \mu$ lowering the prescribed safety threshold to $\mu<1$. By doing so, we will retain the ability of rejecting all unsafe machines almost surely, while explicitly bounding the handicap. With this goal in mind, we present our modified SPRT in the next section.
\end{remark}

\section{Sequential probability ratio test}\label{sec:sprt}
Consider in this section the case of a single machine  with unknown mean $\mu_n$. We face the problem of deciding whether the machine is unsafe, i.e., $\mu_n\leq\mu<1$. 
For this purpose, we set the following Hypothesis test
\begin{equation}
    \begin{cases}
    H_0:&\mu_n\geq \mu+\epsilon\\
    H_1:&\mu_n\leq \mu
    \end{cases}
\end{equation}
where $\epsilon\leq 1-\mu$ is a slack parameter.
The goal of this section is to devise a sequential test, which uses data $X_t$ for $t=1,2,\ldots$ to detect if the machine is unsafe. We look for a test that detects such a machine almost surely,  and that guarantees that a machine with $\mu_n\geq \mu+\epsilon$ is kept with probability $1-\alpha$ as $t$ grows unbounded. The three values $\mu$, $\epsilon$ and $\alpha$ are design parameters. 
An overly  conservative choice, $\epsilon\simeq 0$ $\alpha\simeq 0$ pays the price of a longer detection time, as shown in Lemma \ref{lemma:detection_time} later in this section.
%
The construction of the following test and the analysis of cylinder sets are based on Wald's celebrated  SPRT \cite{wald1945}.


 \begin{figure}[t]
 \centering
\includegraphics[width=0.9\linewidth]{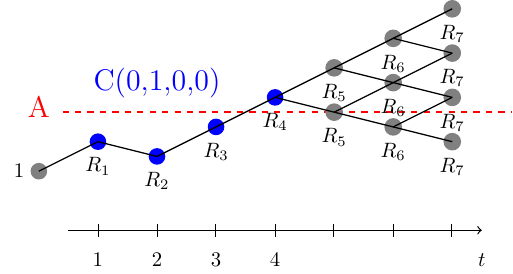}
 \caption{Schematic of the decision rule for the one-sided SPRT \eqref{eq:sprt1}. Sequences $\{x_t\}_{t=1}^\infty$  belonging to the cylinder set $C(0,1,0,0)$ coincide in  the  first $t=4$ entries $(x_1,x_2,x_3,x_4)=(0,1,0,1)$, which determine the likelihood ratios $(L_1,L_2,L_3,L_4)$ represented by blue points.  Since they cross the decision threshold $A$ at $T=4$, the  null hypothesis is rejected and the machine is declared unsafe. The decision is made at $T=4$, and therefore the multiple  possible trajectories of  $L_t$ afterwards are inconsequential and can be disregarded.}
     \label{fig:sprt_schematic}
\end{figure}

Let $f_\mu(x)$ and $f_{\mu+\epsilon}(x)$ denote the probability mass functions corresponding to the Bernoulli distributions of parameters  $\mu$ and $\mu+\epsilon$, respectively.
For $t\in\mathbb{N}$ consider a sampled trajectory $(x_1,\ldots,x_t)$ and define the likelihood ratio:
\begin{equation}
    L_t=\frac{f_\mu\left(x_1,x_2,\ldots,x_t\right)}{f_{\mu+\epsilon}\left(x_1,x_2,\ldots,x_t\right)}\label{eq:probability_ratio}
\end{equation}

At each time we calculate $L_t$ and accept $H_1$ if 
\begin{equation}
    L_t\geq A
    \label{eq:sprt1}
\end{equation}
where the threshold  $A$ is a design parameter that will be specified later. This is, if $L_t\geq A$ we declare that the machine is unsafe and stop the test.  Otherwise  we take an additional observation pulling the arm one more time to then check the condition \eqref{eq:sprt1} again after updating $t\to t+1$. {Assuming i.i.d. samples,  \eqref{eq:sprt1} can be transformed into a condition on the   number of zeros $k_t$ in the trajectory  $(x_1,x_2,\ldots,x_t)$. 
Indeed, taking the logarithm  of \eqref{eq:probability_ratio}, \eqref{eq:sprt1}  transforms into  $k_t \lambda_0 - (t-k_t) \lambda_1 \geq \log A$, with  $\lambda_0:=\log (f_\mu(0)/f_{\mu+\epsilon}(0))$ and $\lambda_1:=-\log (f_{\mu}(1)/f_{\mu+\epsilon}(1))$. Rearranging terms we arrive to  the equivalent condition for \eqref{eq:sprt1}}
\begin{equation}
    k_t\geq \frac{\log A+ \lambda_1 t}{\lambda_0+\lambda_1}
    \label{eq:binomial}
\end{equation}
with 
\begin{align}
    \lambda_0&=\log((1-\mu)/(1-\mu-\epsilon))\\
    \lambda_1&=\log((\mu+\epsilon)/\mu)
    \label{eq:definitionL0L1}
\end{align}

{In this new form, it is apparent that our decision rule reduces to a sequential binomial test. Different from Algorithm \ref{alg:assured_explorer}, \eqref{eq:binomial} does not discard a machine on the first zero, but has a probabilistic rule to decide when the number of zeros does not correspond with the hypothesis of a safe machine. }

There are three questions that we would like to answer in this setting: i) do unsafe machines produce sequences that escape the threshold $A$ with probability one?, ii) do safe machines produce sequences that do not escape this threshold, and if so, with what probability?, and  iii) what is the expected time for the probability ratio of  sequences coming from an unsafe machine  to cross the threshold $A$?
To answer these questions we need some definitions first.

\subsection{Cylinder sets}

Consider an infinite sequence $\{x_\tau\}, \tau=1,2,\ldots$ and define $C_\infty$ as the space of all such  sequences. The set $C(a_1,\ldots,a_t)$ is called a cylinder set of order $t$, and is defined as the subset of $C_\infty$ which collects sequences with  $x_1=a_1, \ldots, x_t=a_t$.
A cylinder set will be said to be of the unsafe type if 

\begin{equation}\label{eq:define_cylinder1}
L_t=\frac{f_{\mu}\left(a_{1},\ldots,a_{t}\right)}{f_{\mu+\epsilon}\left(a_{1},\ldots,a_{t}\right)} \geq A
\end{equation}
and if for all $\tau=1, \cdots, t-1$.
\begin{equation}\label{eq:define_cylinder2}
L_\tau=\frac{f_{\mu}\left(a_{1},\ldots,a_{\tau}\right)}{f_{\mu+\epsilon}\left(a_{1}\ldots, a_{\tau}\right)}<A 
\end{equation}

The first condition ensures that all infinite sequences   $\{x_\tau\}$ with $x_1=a_1,\ldots,x_t=a_t$ will lead to the acceptance of hypothesis $H_1$, thus declaring the machine  unsafe.  This is depicted in Figure \ref{fig:sprt_schematic}, showing a sequence that belongs to an unsafe cylinder of order $t=4$.  Notice that by  construction, a machine that produces a sequence belonging to an unsafe cylinder  as in \eqref{eq:define_cylinder1} will be declared unsafe  at time $t$, regardless of the  future samples $x_\tau$, $\tau>t$. Thus, we can effectively stop the test for that machine at time $t$.  The second condition \eqref{eq:define_cylinder2} ensures that cylinders of different orders are disjoint sets, since the probability ratio must exceed the threshold $A$ for the first time at  $t$ an this cannot be true for two different values of $t$.
The union of all unsafe cylinder sets (of any order) defines the set of sequences that lead to deciding $H_1$.
 Let us  name this (disjoint) union as $Q_U$. 
%
%
%
%
%
%
%
 Let us also define  $Q_S$  as the complement of $Q_U$
\begin{equation}
    Q_S = Q_U^\complement
\end{equation}

This definition means that $Q_S$ is the set of all sequences for which the likelihood ratio $L_t$ never rises above $A$. Because they are  complementary, it holds for all $\mu_n\in[0,1]$
\begin{equation}
P_{\mu_n}(Q_U+Q_S)=1 
\label{eq:pQ1S}
\end{equation}
with   $P_{\mu_n}(Q)$ being  the probability measure corresponding to a Bernoulli distribution of parameter $\mu_n$.


We would like to obtain the following behavior:
\begin{itemize}
    \item Under $H_0$, \emph{most} sequences belong to $Q_S$
    \item Under $H_1$, \emph{all} sequences belong to $Q_U$.
\end{itemize}

The second claim is guaranteed by the following lemma, which  departs from \cite{wald1945} because there is a non-zero probability of not stopping the test, and thus requires a special treatment, focusing on the mean of the estimator through the KL-divergence instead of on its variance. 

\begin{lemma}\label{lemma:unsafe_sequences_escape}
Let  $P_{\mu_n}(Q)$ and  $f_{\mu_n}(x)$ be the probability measure and mass function corresponding to a Bernoulli distribution of parameter $\mu_n$ under the alternative hypothesis $H_1$ ($\mu_n\leq \mu$) . Then, i.i.d. sequences   produced by such a distribution are correctly classified almost surely, that is
\begin{equation}
    P_{\mu_n}(Q_U)=1,\quad \forall \ \mu_n\leq \mu
    \label{eq:p1Q1}
\end{equation}
\end{lemma}
\begin{proof}
For a given sequence $(x_1,\ldots,x_t)$ define the log-likelihood ratio as
\begin{equation}
\Lambda_t = \log L_t=\log\prod_{i=1}^t\frac{f_\mu(x_i)}{f_{\mu+\epsilon}(x_i)}=\sum_{i=1}^t\log\frac{f_\mu(x_i)}{f_{\mu+\epsilon}(x_i)}
\label{eq:log_likelihood}
\end{equation}

Dividing by $t$:
$$
\frac{\Lambda_t}{t}=\frac{1}{t}\sum_{i=1}^t\log\frac{f_\mu(x_i)}{f_{\mu+\epsilon}(x_i)}
$$

Taking the limit as $t\rightarrow\infty$ the above expression converges to the expectation of the right hand side under the alternative hypothesis
\begin{align}
\frac{\Lambda_t}{t}&\rightarrow \mathbb{E}_{x\sim f_{\mu_n}}\left[\log\frac{f_\mu(x)}{f_{\mu+\epsilon}(x)}\right]\nonumber\\
&=\mu_n \log\frac{\mu}{\mu+\epsilon}+(1-\mu_n)  \log\frac{1-\mu}{1-\mu-\epsilon}\nonumber\\
&=\mu_n \left(\log\frac{\mu}{\mu+\epsilon}-\log\frac{1-\mu}{1-\mu-\epsilon}\right)+\log\frac{1-\mu}{1-\mu-\epsilon}  \nonumber\\
&\geq\mu \left(\log\frac{\mu}{\mu+\epsilon}-\log\frac{1-\mu}{1-\mu-\epsilon}\right)+\log\frac{1-\mu}{1-\mu-\epsilon}  \nonumber\\
&= \mu \log\frac{\mu}{\mu+\epsilon}+(1-\mu)  \log\frac{1-\mu}{1-(\mu+\epsilon)}\nonumber\\
&=D_{KL}(f_\mu\mid\mid f_{\mu+\epsilon})>0\label{eq:DKL} 
\end{align}
where  $D_{KL}$ stands for the Kullback-Leibler divergence. The inequality holds because {$\mu_n\leq\mu$} by hypothesis, and it multiplies the expression in brackets, which is negative.  
%
%
From the  inequality in \eqref{eq:DKL} {and the Law of Large Numbers,} it follows
$$
\lim_{t\rightarrow\infty}\Lambda_t = \infty,\quad a.s.
$$
Therefore there must exist a positive integer $t$ for which $\Lambda_t$ exceeds $\log A$, so that the sequence $\{x_\tau\}$ belongs to an  unsafe cylinder of order $t$ and thus $\{x_\tau\}\in Q_U$.
\end{proof}

The previous Lemma proved that  unsafe machines are detected with probability one. Next we prove that, by designing the threshold $A$ judiciously, a fraction $1-\alpha$ of the safe machines are kept in the system indefinitely. Later, in Lemma \ref{lemma:detection_time}, we provide a bound on the expected time it takes to detect an unsafe machine.

\begin{lemma}\label{lemma:safe_error_probability}
Let $A=\frac{1}{\alpha}$ and $\mu_n\geq \mu+\epsilon$. Then,  the probability that a trajectory never rises above $A$ is $P_{\mu_n}(Q_S)\geq 1-\alpha$.
\end{lemma}

\begin{proof}

%
First, we prove the claim $P_{\mu+\epsilon}(Q_S)\geq 1-\alpha$   for the limiting case $\mu_n=\mu+\epsilon$, and then we generalize it for $\mu_n\geq \mu+\epsilon$.
For $\mu_n=\mu+\epsilon$, the core of the proof relies in showing that 
 \begin{equation}
 L_t\geq A \Longrightarrow P_\mu(Q_U) \geq A P_{\mu+\epsilon} (Q_U) \label{eq:pQ1}
 \end{equation}

 Once we prove \eqref{eq:pQ1}, we use the fact that $P_\mu(Q_U)=1$ (see  \eqref{eq:p1Q1}), hence
 \begin{equation}
 P_{\mu+\epsilon}(Q_U)\leq \frac{1}{A} = \alpha \Rightarrow P_{\mu+\epsilon}(Q_S) \geq 1- \alpha 
 \label{eq:P0Qs}    
 \end{equation}
 as desired. 



 To prove \eqref{eq:pQ1}, we start by decomposing $Q_U$ as the union across time $t$ of the union of all unsafe cylinders of order $t$, that is  $$
 Q_U = \bigcup_{t=1}^\infty\bigcup_{(a_1,\ldots,a_t) \in \mathcal A_t} C(a_1,\ldots,a_t)
 $$
 where $\mathcal A_t$ collects the tuples  $(a_1,\ldots,a_t)$ that define  unsafe cylinders of order $t$, i.e., those  satisfying \eqref{eq:define_cylinder1} and \eqref{eq:define_cylinder2}.  
 By construction all the cylinder sets are disjoint, hence 
 \begin{align*}
 P_\mu(Q_U)&=\sum_{t=1}^\infty\sum_{(a_1,\ldots,a_t) \in \mathcal A_t} P_\mu\left(C(a_1,\ldots,a_t)\right)\\
 &=\sum_{t=1}^\infty\sum_{(a_1,\ldots,a_t) \in \mathcal A_t} f_\mu(a_1,\ldots,a_t)\\
 &\geq \sum_{t=1}^\infty\sum_{(a_1,\ldots,a_t) \in \mathcal A_t} A f_{\mu+\epsilon}(a_1,\ldots,a_t)\\
 &=A\sum_{t=1}^\infty\sum_{(a_1,\ldots,a_t) \in \mathcal A_t} P_{\mu+\epsilon}(C(a_1,\ldots,a_t))\\
 &=AP_{\mu+\epsilon}(Q_U)
 \end{align*}
 where second identity follows from marginalizing over future trajectories (see Fig. \ref{fig:sprt_schematic}), and the inequality holds since $\mathcal A_t$ is defined to satisfy \eqref{eq:define_cylinder1}.
 
 Now that we have \eqref{eq:pQ1},  \eqref{eq:P0Qs} follows immediately, and we need to prove that $P_{\mu_n}(Q_S)\geq P_{\mu+\epsilon}(Q_S)$,  or equivalently $P_{\mu_n}(Q_U)\leq P_{\mu+\epsilon}(Q_U)$, when $\mu_n\geq \mu+\epsilon$.   
 This is intuitively true, since  for a sequence to belong to $Q_U$, it must satisfy \eqref{eq:binomial}. But the number of zeros is distributed  $k_t\sim \text{Binomial}(t, 1-\mu_n)$, so that the  probability of  satisfying \eqref{eq:binomial} becomes lower as   $\mu_n$ increases. For a more rigorous proof, we decompose again 
  \begin{equation}
     P_{\mu_n}(Q_U)=\sum_{t=1}^\infty\sum_{(a_1,\ldots,a_t) \in \mathcal A_t} f_{\mu_n}(a_1,\ldots,a_t)\label{eq:PmunQu}
 \end{equation}
 as we did for $P_{\mu}(Q_U)$. We will prove that for any fixed tuple $(a_1,\ldots,a_t) \in\mathcal A_t$, $f_{\mu_n}(a_1,\ldots,a_t)$ is a decreasing function of $\mu_n$. First, we need to prove that the number of zeros $k_t$ in $(a_1,\ldots,a_t)$ satisfies $1-k_t/t\leq \mu_n$.  But because $k_t$ in  $(a_1,\ldots,a_t)$ must satisfy \eqref{eq:binomial}, and $\log A$, $\lambda_0$, and $\lambda_1$ are strictly positive, then $k_t\geq \lambda_1 t /(\lambda_0+\lambda_1)$ or equivalently $k_t/t\geq\left(1+\frac{\lambda_0}{\lambda_1}\right)^{-1}$. From the definition of $\lambda_0$ and $\lambda_1$ it yields
 \begin{align}
     \frac{k_t}{t}&\geq\left(1+\frac{\lambda_0}{\lambda_1}\right)^{-1}=\left(1+\frac{\log[(1-\mu)/(1-\mu-\epsilon)]}{{\log[(\mu+\epsilon)/\mu]}}\right)^{-1}\nonumber\\
     &\geq \left(1 + \frac{(1-\mu)/(1-\mu-\epsilon)-1}{1-\mu/(\mu+\epsilon)}\right)^{-1} \label{eq:log_ineq}\\
     &=\left(\frac{1}{1-\mu-\epsilon}\right)^{-1}= 1-\mu-\epsilon \geq 1-\mu_n
 \end{align}
 where the inequality in \eqref{eq:log_ineq} follows from the usual bounds of the logarithm $1-1/x\leq \log(x)\leq x-1$. Rearranging $1-\frac{k_t}{t}\leq \mu_n$, it results $t-k_t-t\mu_n\leq 0$. In this case, the derivative of $f_{\mu_n}(a_1,\ldots,a_t)$ takes the form
 \begin{align}
    &\frac{d}{d{\mu_n}} f_{\mu_n}(a_1,\ldots,a_t)=\frac{d}{d\mu_n}\mu_n^{t-k_t}(1-\mu_n)^{k_t}\nonumber\\
    &=\mu_n^{t-k_t-1}(1-\mu_n)^{k_t-1}(t-k_t-t \mu_n)\leq 0\label{eq:derivada}
 \end{align}
 Putting  \eqref{eq:pQ1S}, \eqref{eq:P0Qs}, \eqref{eq:PmunQu}, and \eqref{eq:derivada} together results in $P_{\mu_n}(Q_S)\geq 1-\alpha$ for all $\mu_n\geq\mu+\epsilon$. 
  \end{proof}



We have answered two of the three questions about our one-sided SPRT. Once we know that all unsafe machines are detected with probability one, it remains to characterize the detection time, which is the purpose of the following lemma. Henceforth, we will set the decision  threshold at $A=\frac{1}{\alpha}$. 

\begin{lemma}\label{lemma:detection_time} Under the alternative hypothesis $H_1$ corresponding to $\mu_n\leq \mu$, and with $A=\frac{1}{\alpha}$, the test \eqref{eq:sprt1} is expected to terminate after $T$ steps, with \begin{equation}
    \mathbb{E}[T]\leq 1+\frac{\log\left(1/\alpha\right)}{D_{KL}(f_\mu \mid\mid f_{\mu+\epsilon})}
\end{equation}
\end{lemma}
\begin{proof} 

Let $T$ be the smallest integer for which the test leads to the acceptance of $H_1$. Such variable is well defined and finite as a result of Lemma \ref{lemma:unsafe_sequences_escape}.

\begin{align}
    \mathbb{E}_{x\sim f_{\mu_n}}\left[\Lambda_T\right] 
    &= \mathbb{E}_T\left[\mathbb{E}_{x\sim f_{\mu_n}}\left[\sum_{i=1}^T \log\left.\frac{f_\mu(x_i)}{f_{\mu+\epsilon}(x_i)}~\right|~ T\right]\right]\nonumber\\
    &= \mathbb{E}_T\left[T \mathbb{E}_{x\sim f_{\mu_n}}\left[ \log\frac{f_\mu(x)}{f_{\mu+\epsilon}(x)}~\right]\right]\nonumber\\
    &= \mathbb{E}[T] \mathbb{E}_{x\sim f_{\mu_n}}\left[ \log\frac{f_\mu(x)}{f_{\mu+\epsilon}(x)}\right]
    =\mathbb{E}[T] R_n \label{eq:nbounded_1} 
\end{align}
with $R_n=\mathbb{E}_{x\sim f_{\mu_n}}\left[ \log\frac{f_\mu(x)}{f_{\mu+\epsilon}(x)}\right]$.
%
%
%
%
Furthermore,

\begin{align}
\mathbb{E}_{x\sim f_{\mu_n}}\left[\Lambda_T\right] &= \mathbb{E}_{x\sim f_{\mu_n}}\left[\Lambda_{T-1}+\log\frac{f_\mu(x_T)}{f_{\mu+\epsilon}(x_T)}\right]\nonumber\\
&= \mathbb{E}_{x\sim f_{\mu_n}}\left[\Lambda_{T-1}\right]+R_n\leq \log(1/\alpha) + R_n \label{eq:nbounded_2}
\end{align}


Combining \eqref{eq:nbounded_1} and \eqref{eq:nbounded_2}:

$$
\mathbb{E}[T]\leq  1+\frac{\log\left(1/\alpha\right)}{R_n}\leq  1+\frac{\log\left(1/\alpha\right)}{D_{KL}(f_\mu \mid\mid f_{\mu+\epsilon})}
$$

in virtue of $R_n \geq D_{KL}(f_\mu \mid\mid f_{\mu+\epsilon})$ for $\mu_n\leq \mu$, as it was proved in \eqref{eq:DKL}. 
\end{proof}

The result in Lemma \ref{lemma:detection_time} evidences the need of some slack $\epsilon$ between the limiting distributions of both hypotheses. By accommodating this gap we are able to separate  the limiting distributions $f_{\mu}(x)$ from $f_{\mu+\epsilon}(x)$ so that the distance $D_{KL}(f_\mu \mid\mid f_{\mu+\epsilon})$ is positive and we can guarantee a finite expected detection time. Notice that we could use   $R_n$ for the bound on the detection time, as given in  the proof of Lemma \ref{lemma:detection_time}, which indeed  gives a tighter bound, meaning  faster detection. However, it requires the knowledge of the underlying probability $\mu_n$ which is unknown. Using the $D_{KL}$ instead is preferred, because it yields a  bound that depends on our design parameters $\mu$, $\epsilon$ and $\alpha$ only. This result also allows  us to have  some intuitive interpretation regarding the choice of these design parameters.

With lemmas \ref{lemma:unsafe_sequences_escape}--\,\ref{lemma:detection_time} at hand, we return to our original  setup.

\section{Learning to be safe}\label{sec:relaxed}
Next we generalize  Algorithm \ref{alg:assured_explorer} for the case in which the safety requirement $\mu=1$ is relaxed. The following algorithm results from extending the one-sided SPRT just described to the scenario with multiple machines. Identical to the previous Section, we prescribe a safety threshold $\mu<1$ that renders machines with $\mu_n<\mu$ as unsafe. Then, we define an error probability $\alpha$, a slack parameter $\epsilon$, and the Bernoulli probability mass functions    $f_{\mu+\epsilon}(x)$ and $f_\mu(x)$ with means $\mu+\epsilon$ and $\mu$ respectively. These are all the definitions needed to run  our second Safety Inspector  algorithm.

\begin{center}
\begin{algorithm}
Initialize $\mathcal{S}_1=\left\{1,\ldots,N\right\}$,\ $\Lambda_n=0,\ n=1,\ldots,N$\\
\For{$t=1,2\ldots$}{
Pick an arm $I_t \sim$ Unif($\mathcal{S}_t$)\\
Observe return $X_t=X_{I_t,t}$\\
Update $\Lambda_{I_t}+=\log\frac{f_\mu(X_t)}{f_{\mu+\epsilon}(X_t)}$\\

\If{$\Lambda_{I_t}\geq \log(1/\alpha)$}{$\mathcal{S}_t \leftarrow \mathcal{S}_{t-1}\setminus \{I_t\}$
}

}
\caption{Relaxed Safety Inspector}
\label{alg:relaxed_assured_explorer}
\end{algorithm}
\end{center}

Building on lemmas \ref{lemma:handicap_tnt}--\ref{lemma:detection_time} we state our main result.

\begin{theorem}
The  handicap and safety ratio of the Relaxed Safety Inspector (Algorithm \ref{alg:relaxed_assured_explorer}) satisfy
\begin{align}\mathbb{E}[\handicap_t]&\leq M\left(1+\frac{\log\left(1/\alpha\right)}{D_{KL}(f_\mu \mid\mid f_{\mu+\epsilon})}\right)\label{eq:expected_handicap_bound}\\
    \mathbb{E}[\rho_t]&\geq 1-\alpha \label{eq:rho_bound}
\end{align}
and the testing time of unsafe machines is bounded by 
\begin{align}
\mathbb{E}[T_n(t)]&\leq 1+\frac{\log\left(1/\alpha\right)}{D_{KL}(f_\mu \mid\mid f_{\mu+\epsilon})} \label{eq:tnt_relaxed}
\end{align}
\label{prop:relaxed_assuredRL}

\end{theorem}
\begin{proof}
The third inequality  was proved in Lemma \ref{lemma:detection_time}. Notice that $T_n(t)$ is defined as the number of trials for machine $n$, regardless of the time spent on other machines,   so that we can treat it as the result of separate SPRTs, and thus use Lemma \ref{lemma:detection_time}.
The first inequality follows from \eqref{eq:tnt_relaxed} and Lemma \ref{lemma:handicap_tnt}. The second one results  from   Lemma \ref{lemma:safe_error_probability}. 
\end{proof}
Theorem \ref{prop:relaxed_assuredRL} certifies that the Relaxed Safety Inspector (Algorithm \ref{alg:relaxed_assured_explorer}) inherits the finite detection time from the SPRT, ensuring that all unsafe machines are removed in finite time, providing a universal bound  \eqref{eq:tnt_relaxed} in terms of the design parameters $\alpha, \mu$ and $\epsilon$. As a consequence, the total handicap of the system also remains bounded by a finite constant \eqref{eq:expected_handicap_bound}. Together with the certainty of rejecting all  unsafe machines, our algorithm ensures that a proportion $1-\alpha$ of the safe machines with slack $\epsilon$ is kept in the system indefinitely according to \eqref{eq:rho_bound}.    








\section{Numerical examples}\label{sec:numerical}

Let us first illustrate the behavior of Algorithm \ref{alg:relaxed_assured_explorer} and how it proceeds to discard unsafe arms. To that end, consider a simple setup of $N=3$ arms, all with same  parameter $\mu_n=0.8$. Our safety requirements are set to $\mu=0.9$, $\epsilon=0.02$ and the error probability to $\alpha=0.05$. With this parameters in mind, all machines should be deemed unsafe in finite time. Our discarding rule involves checking when $\Lambda_t\geq\log A$. Since we are dealing with Bernoulli random variables, this rule can be equivalently cast as a decision based on the number of  failed outcomes of each arm \eqref{eq:binomial}. These reciprocal ideas are depicted on the test in Figure \ref{fig:sprt}.

\begin{figure}[h]
    \centering
    \includegraphics[width=.9\linewidth]{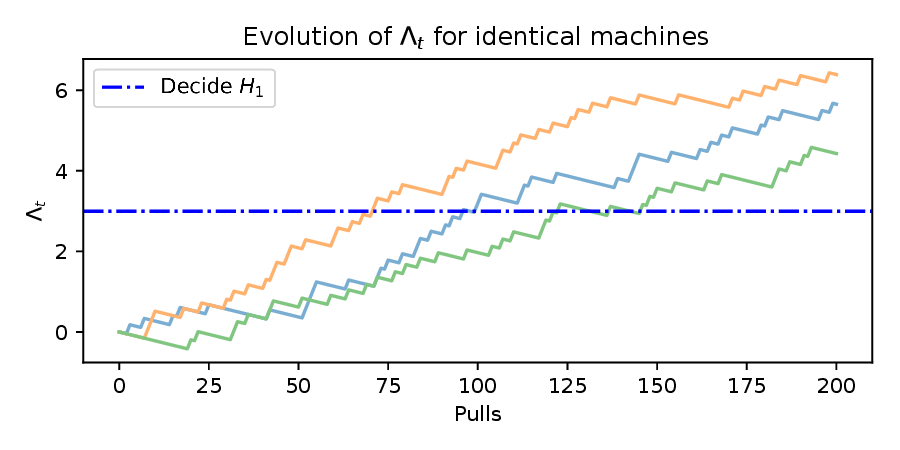}
    \includegraphics[width=.9\linewidth]{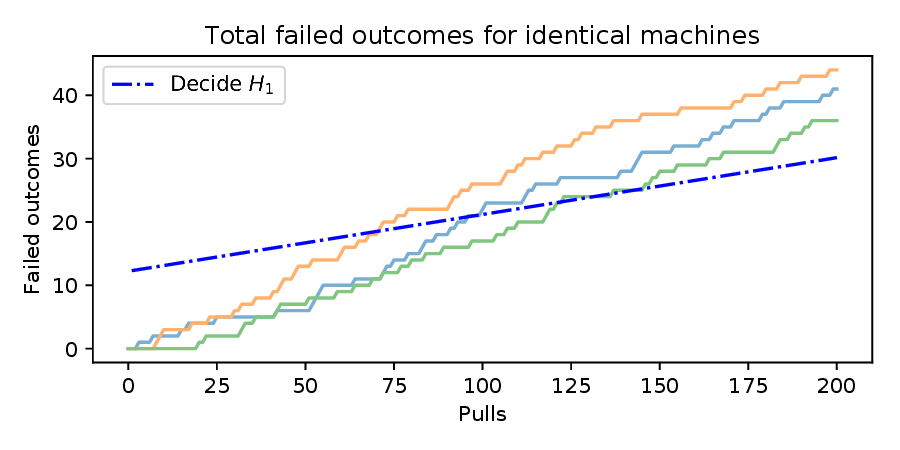}
    \caption{Sequential probability ratio test for $\mu=0.9$, $\epsilon=0.02$, $\alpha=0.05$ and three identical unsafe machines of parameter $\mu_n=0.8$. Above: log-likelihood function $\Lambda_t$ as a function of the number of pulls of each machine. Below: failed outcomes for each machine as a function of the number of pulls. The test terminates when the sequence surpasses either rejection line (in blue). All three unsafe machines are discarded in finite time.}
    \label{fig:sprt}
\end{figure}

\subsection{Experiment 1: Transient behavior}

We consider a setup of $N=1000$ arms, and set a safety guarantee $\mu = 0.9$ and a gap $\epsilon = 0.05$. The true parameter of each arm is sampled from a uniform distribution $\mathcal{U}\left(0.8,1\right)$. We run sixteen instances of the Safety Inspector described in Algorithm \ref{alg:relaxed_assured_explorer} on this test-bed, and then average the results obtained.
We define the \emph{normalized handicap} as the average handicap over all the available arms: $\text{NHandicap}_t=\frac{1}{N}\text{Handicap}_t$.
Figures \ref{fig:handicap_t} and \ref{fig:rho_t} show the evolution of the normalized handicap and safety ratio for different tolerance levels $\alpha$, along with the bounds obtained in Theorem  \ref{prop:relaxed_assuredRL}. Notice that both the handicap and safety ratio remain constant after some time, which indicates that all unsafe machines have been identified, and that no more safe machines are discarded along the way.
The final handicap obtained is essentially the number of pulls over all unsafe arms. This is depicted more closely in Figure \ref{fig:hist}, which presents a histogram of the \emph{testing time} on unsafe machines, for the setup explained above and for fixed $\alpha=0.1$. Most machines yield a testing time that is strictly lower than the bound in \eqref{eq:ExpectedT}. It is important to remark that this bound is on the \emph{expected} time, and therefore a small number of machines actually need to be tested for longer. Nevertheless, the empirical mean represented by a dashed green line in Fig. \ref{fig:hist} lies below the red line that represents the bound. This bound, as well as those in Figs. \ref{fig:handicap_t} and \ref{fig:rho_t}, is loose because the machines parameters $\mu_n$ were drawn uniformly  from $\mathcal{U}\left(0.8,1\right)$, but becomes tight if selected as the limiting parameters  of the hypothesis test $\mu_n\in\{\mu,\mu+\epsilon\}$.
\begin{figure}[h]
    \centering
    \includegraphics[width=.9\linewidth]{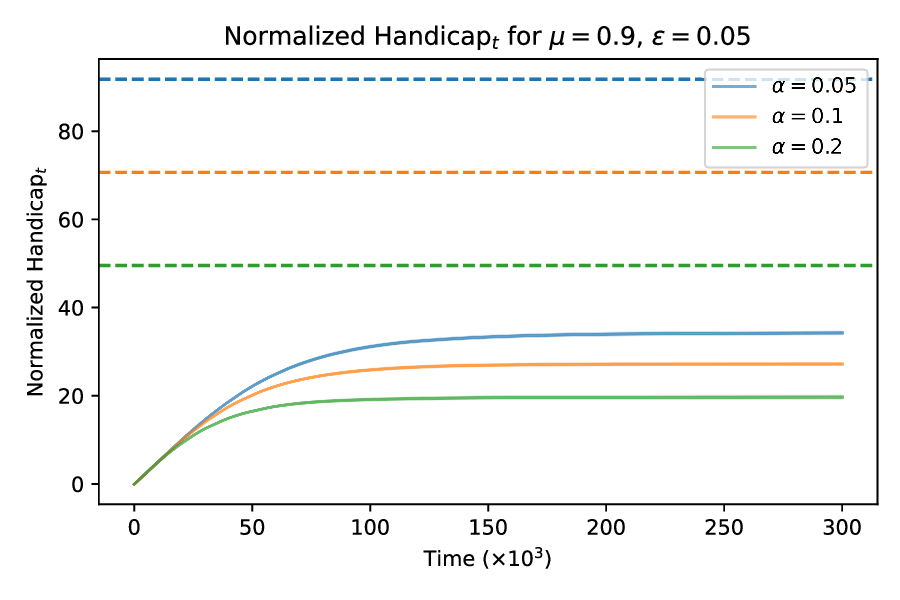}
    \caption{Evolution of the Normalized Handicap through training for $\mu=0.9$, $\epsilon=0.05$ and machines with parameter drawn from $\mathcal{U}(0.8,1)$. Each solid line corresponds to the handicap obtained with a different error tolerance $\alpha$, and the dashed lines are the (normalized) bounds on the Handicap (see \eqref{eq:expected_handicap_bound}). All unsafe machines are eventually discarded after sufficient training, and therefore the handicap remains constant.}
    \label{fig:handicap_t}
\end{figure}

\begin{figure}[h]
\centering
    \includegraphics[width=.9\linewidth]{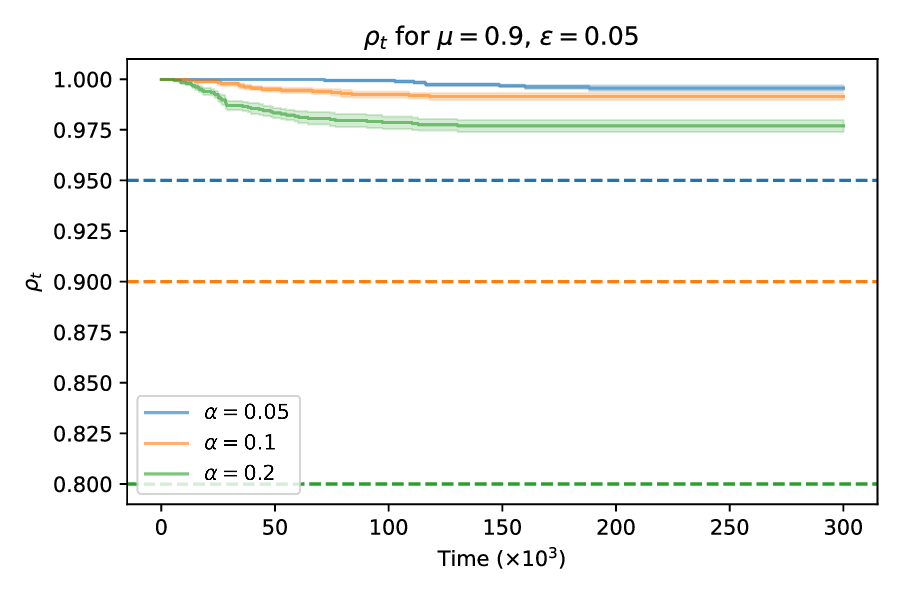}
    \caption{Evolution of the safety ratio $\rho_t$ through training for $\mu=0.9$, $\epsilon=0.05$ and machines with parameter drawn from $\mathcal{U}(0.8,1)$, for different tolerance levels $\alpha$. Each solid line is accompanied by its corresponding bound (see \eqref{eq:rho_bound}).}
    \label{fig:rho_t}
\end{figure}

\begin{figure}[h]
\centering
    \includegraphics[width=.9\linewidth]{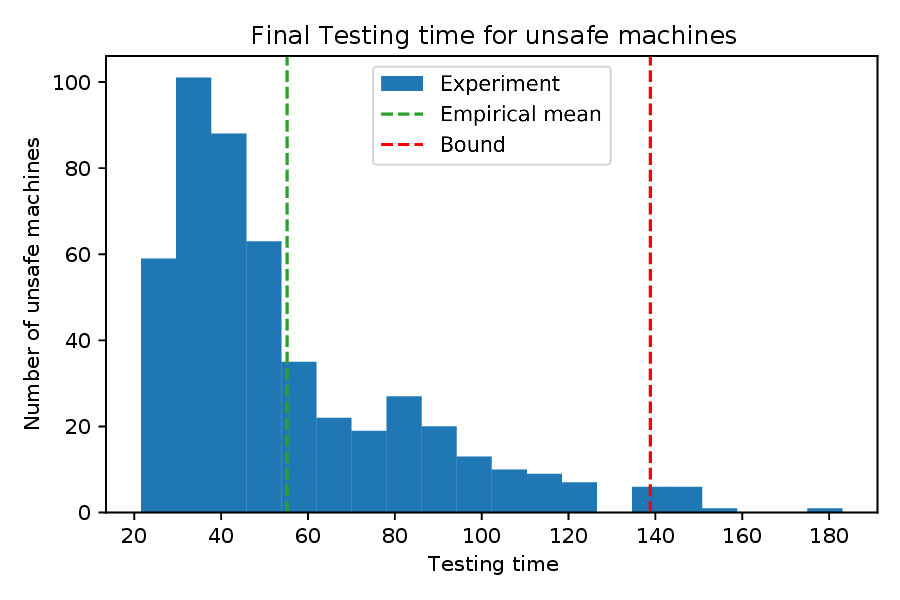}
    \caption{Histogram of the testing time needed to discard unsafe machines, for $\mu=0.9$, $\epsilon=0.05$, $\alpha=0.1$ and machines' parameters drawn from $\mathcal{U}(0.8, 1)$. In dashed red: bound on the testing time from Theorem \ref{prop:relaxed_assuredRL}. Since this bound is on the expected testing time, some machines need to be tested for longer. In dashed green: Empirical mean testing time, which is strictly lower than the bound described.}
    \label{fig:hist}
\end{figure}

\subsection{Experiment 2: Steady state behavior}
Repeating the same setup as in \emph{Experiment 1}, now we perform multiple runs for varying $\epsilon$.
Let us define $\text{NHandicap}_\infty$ as the maximal normalized handicap obtained ---that is, the normalized handicap when all unsafe machines are discarded. Figure \ref{fig:stoptime_handicap} illustrates the dependence of $\text{NHandicap}_\infty$ for varying $\epsilon$ and $\alpha$. Larger values of $\alpha$ and $\epsilon$ attain lower handicap, which essentially means that unsafe arms are detected faster. This, however, comes at a price ---faster detection necessarily implies discarding safe machines along the way. Figure \ref{fig:rho} shows the other side of the coin: the final value of the safety ratio $\rho_t$ after all unsafe machines have been discarded, which we dub $\rho_\infty$. As $\epsilon$ and $\alpha$ grow, $\rho_\infty$ diminishes. The conjunction of these two Figures exemplify the preservation-exploration trade-off inherent to our Algorithm.

\begin{figure}[h]
    \centering
  \includegraphics[width=.9\linewidth]{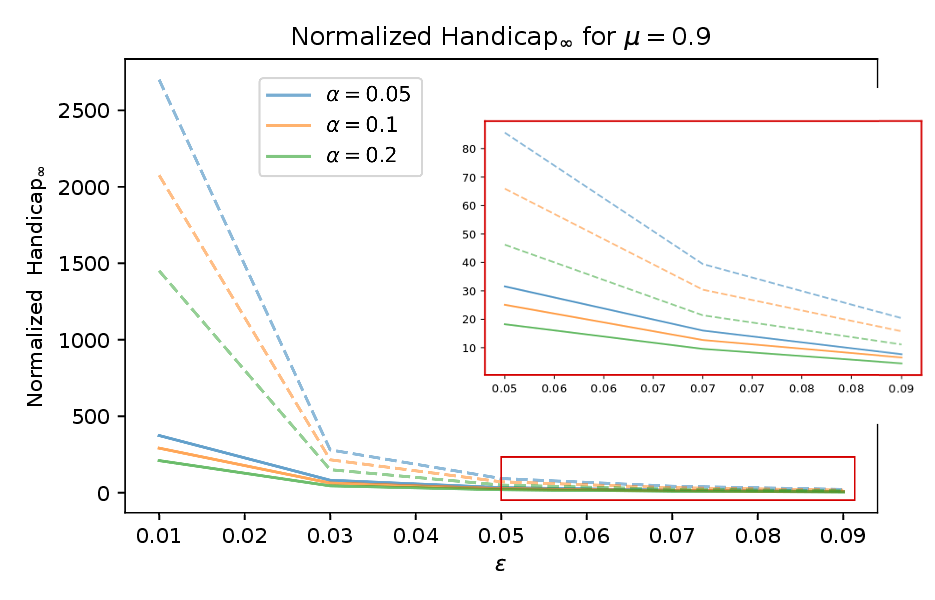}
  \caption{Normalized final Handicap for $\mu=0.9$ as a function of $\epsilon$, for varying tolerance level $\alpha$. As is to be expected, larger values of $\alpha$ and $\epsilon$ achieve lower handicap (which implies faster detection).}
  \label{fig:stoptime_handicap} 
\end{figure}

\begin{figure}[h]
    \centering
    \includegraphics[width=.9\linewidth]{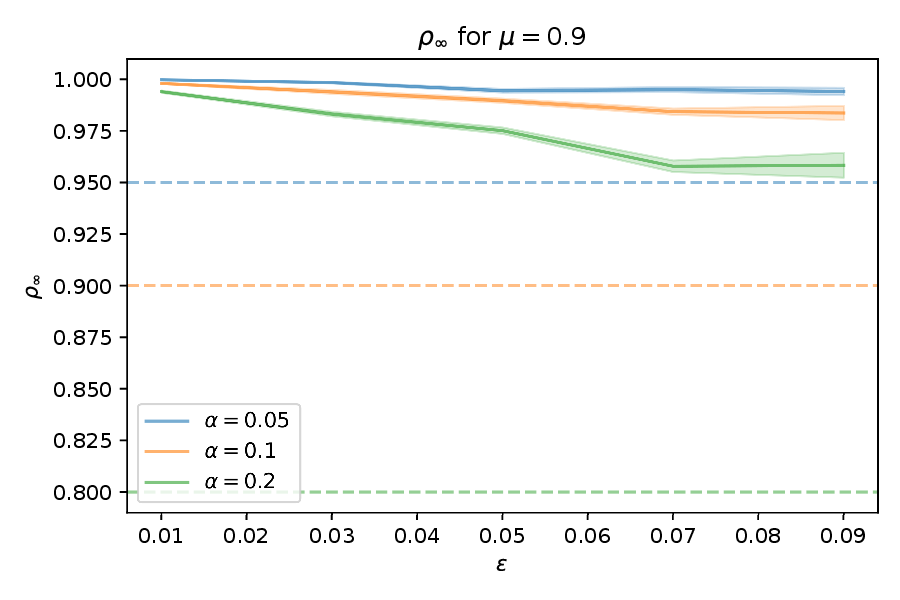}
    \caption{Final safety ratio $\rho_{\infty}$ for $\mu=0.9$ as a function of $\epsilon$, for varying $\alpha$. More machines are kept when using small values of $\epsilon$ and $\alpha$, but this in turn implies longer training time (c.f. Figure \ref{fig:stoptime_handicap})} 
    \label{fig:rho}
\end{figure}

\section{Conclusions}\label{sec:conclusions}
In this paper we are interested in providing a safety environment for learning. To that end, we advance the idea that detecting if an action is safe is much simpler than trying to estimate its value function, and can be done in finite time.  In this direction, we define a measure of handicap that complements the {notion of regret by accounting for the aggregate number of unsafe actions explored. } We focus on the multi-armed bandit problem, with the goal of detecting the malfunctioning machines and keeping the handicap bounded. For this purpose, we  introduced the Relaxed Safety Inspector (Algorithm \ref{alg:relaxed_assured_explorer}), which we developed as  sequential probability ratio test for parallel hypotheses. We proved in Theorem \ref{prop:relaxed_assuredRL} that  this algorithm has the property of removing all unsafe machines in finite time, providing a universal bound  \eqref{eq:tnt_relaxed} in terms of the  the design parameters $\alpha, \mu$ and $\epsilon$. As a consequence of this,  the  handicap  remains bounded by  a finite constant as time goes to infinity. The price to pay for being able to detect all malfunctioning machines in finite time  is to accommodate a slack on the machines that are considered safe, and losing a proportion of them. Interestingly, these are design parameters that can be tightened if we are willing to wait longer for detection.   

\section{Acknowledgements}
The authors thank Hancheng Min and Yue Shen for their comments and valuable suggestions.

% ...




\bibliography{refs}
\bibliographystyle{ieeetr}

\end{document}